\documentclass[11pt]{article}
\usepackage[utf8]{inputenc}
\usepackage[margin=1in]{geometry}
\usepackage{graphicx}
\usepackage{amsmath,amssymb,amsthm}
\newtheorem{proposition}{Proposition}
\usepackage{bm}
\usepackage[pdfencoding=auto,psdextra]{hyperref}
\usepackage{booktabs}
\usepackage{xcolor}
\usepackage{authblk}
\usepackage{tikz}
\usepackage{multirow}
\usetikzlibrary{positioning, arrows.meta, fit, backgrounds, calc, decorations.pathreplacing}

\newcommand{\gap}[1]{{\scriptsize(#1\%)}}
\newcommand{\gappp}[1]{{\scriptsize(#1\,pp)}}

\title{Boltzmann Attention: Learnable Ising Couplings for Cooperative Attention}

\author[1]{Gilhan Kim\thanks{gilhan.kim@yonsei.ac.kr}}
\author[1,2,3]{Daniel K. Park\thanks{dkd.park@yonsei.ac.kr (corresponding author)}}
\affil[1]{Department of Statistics and Data Science, Yonsei University, Seoul 03722, Republic of Korea}
\affil[2]{Department of Applied Statistics, Yonsei University, Seoul 03722, Republic of Korea}
\affil[3]{Department of Quantum Information, Yonsei University, Seoul 03722, Republic of Korea}

\date{}

\begin{document}

\maketitle

\begin{abstract}
Attention mechanisms are central to modern sequence models, yet standard attention computes relevance primarily through individual query--key similarities. Although softmax normalization introduces competition among positions, a standard attention layer does not explicitly parameterize learnable interactions between attention decisions. This limits its ability to directly model cooperative or antagonistic co-attention structure within the attention mechanism itself. We propose Boltzmann attention, an energy-based generalization in which attention patterns are governed by an interacting Ising model. The method augments the usual data-dependent local fields with learnable pairwise couplings, allowing the model to represent inter-position correlations beyond those captured by softmax or sigmoid attention. Experiments on character-level language modeling and synthetic bracket matching show that Boltzmann attention consistently improves over standard softmax attention within a standard Transformer architecture, with the advantage becoming more pronounced as sequence length increases. A four-way ablation confirms that the improvement arises from the learnable pairwise couplings. These results suggest that explicit inter-position interactions provide a principled enhancement for attention-based sequence modeling. Moreover, the Ising formulation opens a natural path toward quantum-computing-based sampling strategies: we demonstrate that diabatic quantum annealing provides a practical training method while maintaining competitive performance with exact Boltzmann computation.

\end{abstract}

\noindent\textbf{Keywords:} Attention mechanism, Energy-based model, Ising model, Boltzmann machine, Transformer, Statistical physics

\section{Introduction}
\label{sec:intro}
Attention mechanisms have become a central primitive of modern artificial intelligence, underlying Transformers and many large-scale language, vision, and multimodal models~\cite{vaswani2017attention,devlin2019bert,dosovitskiy2021image,radford2021learning}. By adaptively routing information across tokens, attention enables flexible context-dependent representations and long-range dependency modeling. While much effort has focused on improving the computational efficiency of attention, its representational structure remains equally important: the way attention parameterizes relationships among positions directly shapes what forms of dependency can be expressed within a layer. In this work, we focus on this representational aspect and ask whether attention can be enriched by explicitly modeling interactions among the attention decisions themselves.

Standard Transformer attention~\cite{vaswani2017attention}, building on the neural attention introduced by Bahdanau et al.~\cite{bahdanau2015neural}, computes attention weights from query--key similarities. For a query vector $\mathbf{q}_i$ at position $i$ and a key vector $\mathbf{k}_j$ at position $j$, the raw attention score is $h_{ij}=\mathbf{q}_i\cdot\mathbf{k}_j/\sqrt{d_k}$,
where $d_k$ is the key dimension. The normalized attention weight assigned from position $i$ to position $j$ is then
\begin{equation}
    \alpha_{ij} = \frac{\exp(h_{ij})}{\sum_l \exp(h_{il})}
    =
    \frac{\exp(\mathbf{q}_i \cdot \mathbf{k}_j / \sqrt{d_k})}{\sum_{l} \exp(\mathbf{q}_i \cdot \mathbf{k}_l / \sqrt{d_k})}.
    \label{eq:softmax_attn}
\end{equation}
The denominator introduces competition among positions: increasing one raw score increases its own normalized attention weight while decreasing the weights assigned to the others.
In statistical physics language, this corresponds to an energy function with local fields $h_{ij}$ but no spin--spin couplings ($J=0$)~\cite{baxter1982exactly}. Thus, softmax attention induces competition through normalization, but not through learnable pairwise couplings between attention decisions. Sigmoid attention~\cite{ramapuram2024sigmoid}, which replaces global normalization with elementwise $\sigma(2h_{ij})$, provides the clean $J=0$ limit: each position responds independently to its local field, with no pairwise coupling between positions.

The absence of pairwise couplings is a structural limitation. In natural language and structured sequence tasks, relevance is often relational: attending to a subject may make its verb more relevant, and attending to an opening bracket may require attending to its matching closing bracket. Multi-head attention does not resolve this limitation: although it runs several independent heads in parallel, each head still computes attention weights position-by-position with no learnable coupling between positions. Stacking multiple layers can partially compensate, since the output of one layer feeds into the next, but this compensation is indirect---the inter-position correlations must be reconstructed through successive layers of representation mixing rather than being explicitly parameterized within the attention mechanism itself.

We address this limitation by formulating attention as an interacting spin system. We assign a binary spin $s_j \in \{-1,+1\}$ to each key position $j$, representing attend ($+1$) or ignore ($-1$). In the Ising model, pairwise couplings $J_{jk}$ between spins create correlations that no external field alone can produce: when $J_{jk} > 0$, attending to position $j$ directly increases the probability of attending to position $k$, and vice versa. The resulting energy function is:
\begin{equation}
    E_i(\mathbf{s}) = -\sum_j h_{ij}\, s_j - \sum_{j<k} J_{jk}\, s_j s_k,
    \label{eq:ising_energy}
\end{equation}
where the local fields $h_{ij}$ come from query--key similarity (as in softmax) and the couplings $J_{jk}$ are learnable parameters encoding inter-position co-attention structure. 
Attention weights are obtained from the marginal spin magnetizations under the corresponding Boltzmann distribution, with the full construction given in Section~\ref{sec:model}.

The connection between attention and statistical physics has been explored from multiple angles. The Hopfield--attention framework~\cite{hopfield1982neural,demircigil2017model,ramsauer2020hopfield} showed that softmax attention corresponds to a one-step energy minimization in a modern Hopfield network, and subsequent works~\cite{hoover2023energy,ota2023attnbm} extended this perspective to iterative energy descent and Boltzmann machines~\cite{ackley1985learning,hinton2002training}. More recently, Poc-L{\'o}pez and Aguilera~\cite{poclopez2024dynamical} applied dynamical mean-field theory from statistical physics to characterize large self-attention networks. Independently, recent work on sigmoid attention~\cite{ramapuram2024sigmoid,yan2025sigmoid} demonstrated that replacing softmax with elementwise sigmoid can match or exceed softmax performance at scale. These lines of research have developed largely in parallel. Physics-based interpretations of softmax attention have remained largely in the non-interacting regime ($J=0$), without introducing learnable couplings, while the sigmoid-attention literature provides strong theoretical and practical advantages, such as improved regularity and hardware efficiency, without an explicit statistical-physics interpretation.

These observations lead to a simple unifying picture. Softmax and sigmoid are structurally different---categorical normalization versus independent binary decisions---but both lie in the $J{=}0$ regime. Boltzmann attention moves beyond this regime by adding learnable inter-position couplings $J\neq0$, creating correlations that neither can represent, analogous to the transition from an ideal paramagnet to an interacting spin system.

Introducing $J$ turns attention into a genuine many-body physics problem. The partition function $Z_i$ involves a sum over $2^T$ spin configurations, where $T$ is the attention window size, that cannot be factorized. Notably, the Ising model underlying Boltzmann attention is precisely the problem class targeted by quantum annealing~\cite{kadowaki1998quantum,johnson2011quantum} and related quantum sampling techniques. Recent work on diabatic quantum annealing (DQA)~\cite{gyhm2024boltzmann,kim2026diabatic,kim2026BMVAE} has demonstrated that Boltzmann machines can be efficiently trained using hardware-native Ising samples, providing a concrete path toward scaling Boltzmann attention beyond regimes that are tractable for classical computation. We return to this possibility in Section~\ref{sec:dqa}, where we demonstrate proof-of-principle DQA-based training.

Our contributions are threefold. First, we propose \emph{Boltzmann attention}, an Ising-based generalization of standard attention that introduces learnable inter-position couplings directly into the attention distribution. Second, we show empirically that learnable couplings improve sequence modeling performance within a standard Transformer architecture, with gains that become more pronounced as sequence length increases: Boltzmann attention improves over softmax by up to $+1.08\%$ in perplexity on Tiny Shakespeare and by $+2.89$\,pp in accuracy on bracket matching at $T=16$. Third, through a four-way ablation comparing softmax, $h{+}J$, $h$-only, and $J$-only variants, we isolate the pairwise couplings $J$ as the source of the improvement. We further demonstrate proof-of-principle DQA-based training, showing that quantum-sampling-based Boltzmann inference can replace exact enumeration while maintaining competitive performance, thereby providing a scalable route beyond the small-$T$ regime.


\section{Related Work}
\label{sec:related}

Several prior works have connected attention to energy-based or Ising models. We organize the comparison along four axes: (i)~whether the model introduces pairwise couplings $J$ between attention variables, (ii)~whether $J$ is freely learnable or derived from other quantities, (iii)~how the resulting distribution is solved, and (iv)~whether the model improves over softmax.

The connection between attention and energy-based models originates in the Hopfield network literature. The original Hopfield network~\cite{hopfield1982neural} introduced energy-based associative memory with binary neurons, and Demircigil et al.~\cite{demircigil2017model} extended it to exponential storage capacity. Ramsauer et al.~\cite{ramsauer2020hopfield} showed that softmax attention is a one-step energy minimization in this modern Hopfield network. Krotov and Hopfield~\cite{krotov2020large} unified dense associative memory models with higher-order interactions through a biologically plausible framework. The Energy Transformer~\cite{hoover2023energy} iterates the energy descent across the full Transformer block. Ota and Karakida~\cite{ota2023attnbm} (AttnBM) recast the Hopfield energy as a Boltzmann machine with exact analytic inference. In the language of our framework, all these models remain in the \emph{non-interacting} regime ($J = 0$): they model interactions between the query and keys, but not between key positions. The $J{=}0$ bottleneck identified in Section~\ref{sec:model} is therefore present in all of them.

Recent work on sigmoid attention~\cite{ramapuram2024sigmoid,yan2025sigmoid} demonstrated that replacing softmax with elementwise sigmoid---$\alpha_{ij} = \sigma(2h_{ij})$---can match or exceed softmax performance at scale, while offering improved regularity and hardware-efficient implementations. These are significant practical advantages that make sigmoid attention an attractive alternative to softmax. In our framework, however, sigmoid attention corresponds precisely to the $J = 0$ Ising model: each spin responds to its local field independently. Thus, while sigmoid removes the categorical constraint of softmax, the absence of learnable pairwise couplings---and hence the inability to parameterize inter-position correlations---remains intact.

A separate line of work reduces the $O(T^2)$ cost of softmax attention through sparsity~\cite{child2019generating}, hashing~\cite{kitaev2020reformer}, or kernel approximations~\cite{katharopoulos2020transformers,choromanski2021rethinking}. These methods modify \emph{which} positions interact but retain the non-interacting attention structure ($J = 0$). Our work addresses a different dimension: we enrich the \emph{nature} of interactions (from independent to cooperative) rather than reducing their computational cost.

Two recent works introduce pairwise couplings into attention but differ from ours in critical ways (Table~\ref{tab:comparison}). QAMA~\cite{du2025qama} reformulates attention as a quadratic unconstrained binary optimization (QUBO) problem with couplings derived from query--key similarity scores and local fields derived from value vectors through a learnable projection. Because $J$ is a fixed function of query--key scores rather than a freely learnable parameter, it cannot encode structural priors beyond what the scores already contain; on image classification benchmarks, QAMA reports accuracy up to 2.7 percentage points below standard softmax attention. The Spin-Model Transformer~\cite{mcbal2023spin} proposes continuous vector spins with couplings computed from query--key transformations rather than learned independently, and uses mean-field inference; no improvement over softmax is reported. Our approach combines all three necessary ingredients: \emph{freely} learnable $J$ (independent of the input), data-dependent $h$ from query--key similarity, and exact inference that preserves correlations.

\begin{table}[t]
    \centering
    \caption{Comparison of energy-based attention models. In the $J$ and $h$ source columns, QK denotes derivation from query--key products, and V denotes derivation from value vectors.}
    \label{tab:comparison}
    \begin{tabular*}{\linewidth}{@{\extracolsep{\fill}}lcccl}
        \toprule
        Model & $J$ source & $h$ source & Inference & Reported result \\
        \midrule
        Hopfield Attn~\cite{ramsauer2020hopfield} & none & QK & analytic & matches softmax \\
        Energy Tr.~\cite{hoover2023energy} & none & QK & iterative & matches softmax \\
        AttnBM~\cite{ota2023attnbm} & none & QK & analytic & matches softmax \\
        Sigmoid~\cite{ramapuram2024sigmoid} & none & QK & independent & $\approx$ softmax at scale \\
        QAMA~\cite{du2025qama} & from QK & from V & QUBO & $-$2.7\,pp vs.\ softmax \\
        Spin-Model~\cite{mcbal2023spin} & from QK & QK & Mean field & no improvement \\
        \textbf{Ours} & \textbf{learnable} & \textbf{QK} & \textbf{exact} & \textbf{+2.89\,pp} (bracket) \\
        \bottomrule
    \end{tabular*}
\end{table}

\section{Boltzmann Attention}
\label{sec:model}

\begin{figure*}[t]
\centering
\begin{tikzpicture}[
    >=Stealth,
    scale=0.78, every node/.style={transform shape},
    block/.style={draw, rounded corners, minimum height=0.9cm, minimum width=2.8cm,
                  font=\small, align=center, fill=white},
    attnblock/.style={block, fill=blue!8},
    boltzblock/.style={block, fill=red!12, draw=red!60!black, thick},
    arrow/.style={->, thick},
]

\node[font=\bfseries] (title_std) at (0, 6.2) {Standard Transformer};

\node[block] (emb_s) at (0, 5.2) {Token + Position Emb};
\node[attnblock] (attn_s) at (0, 3.8) {Attention\\(Softmax)};
\node[block] (ffn_s) at (0, 2.2) {Feed-Forward\\Network};
\node[block] (head_s) at (0, 0.6) {Output Head\\(Softmax)};

\draw[arrow] (emb_s) -- (attn_s);
\draw[arrow] (attn_s) -- (ffn_s);
\draw[arrow] (ffn_s) -- (head_s);

\draw[->, thick, gray!50] (emb_s.east) -- ++(0.6,0) |- ($(attn_s.south)!0.5!(ffn_s.north)$);
\draw[->, thick, gray!50] ($(attn_s.south)!0.5!(ffn_s.north)$) -- ++(2.4,0) |- ($(ffn_s.south)!0.5!(head_s.north)$);

\node[draw, dashed, rounded corners, fill=blue!4, minimum width=4.8cm, minimum height=2.8cm,
      font=\small, align=center] (zoom_s) at (-4.5, 3.8) {};
\node[font=\footnotesize\bfseries, blue!60!black] at (-4.5, 4.9) {Softmax Attention};
\node[font=\footnotesize, align=left] at (-4.5, 4.2) {$h_{ij} = \mathbf{q}_i \!\cdot\! \mathbf{k}_j / \sqrt{d_k}$};
\node[font=\footnotesize, align=left] at (-4.5, 3.6) {$\alpha_{ij} = \mathrm{softmax}_j(h_{ij})$};
\node[font=\footnotesize, red!60!black] at (-4.5, 3.0) {$J = 0$ (no learnable couplings)};

\draw[->, dashed, gray] (attn_s.west) -- (zoom_s.east);

\node[font=\bfseries] (title_b) at (7, 6.2) {Boltzmann Transformer};

\node[block] (emb_b) at (7, 5.2) {Token + Position Emb};
\node[boltzblock] (attn_b) at (7, 3.8) {Attention\\(\textbf{Boltzmann})};
\node[block] (ffn_b) at (7, 2.2) {Feed-Forward\\Network};
\node[block] (head_b) at (7, 0.6) {Output Head\\(Softmax)};

\draw[arrow] (emb_b) -- (attn_b);
\draw[arrow] (attn_b) -- (ffn_b);
\draw[arrow] (ffn_b) -- (head_b);

\draw[->, thick, gray!50] (emb_b.west) -- ++(-0.6,0) |- ($(attn_b.south)!0.5!(ffn_b.north)$);
\draw[->, thick, gray!50] ($(attn_b.south)!0.5!(ffn_b.north)$) -- ++(-2.4,0) |- ($(ffn_b.south)!0.5!(head_b.north)$);

\node[draw, dashed, rounded corners, fill=red!4, minimum width=4.8cm, minimum height=2.8cm,
      font=\small, align=center] (zoom_b) at (11.5, 3.8) {};
\node[font=\footnotesize\bfseries, red!60!black] at (11.5, 4.9) {Boltzmann Attention};
\node[font=\footnotesize, align=left] at (11.5, 4.2) {$h_{ij} = \mathbf{q}_i \!\cdot\! \mathbf{k}_j / \sqrt{d_k}$};
\node[font=\footnotesize, align=left] at (11.5, 3.6) {$E = -\!\sum_j h_{ij} s_j - \!\sum_{j<k} J_{jk} s_j s_k$};
\node[font=\footnotesize, align=left, text=red!70!black] at (11.5, 3.0) {$\alpha_{ij} = (\langle s_j \rangle + 1)/2$};

\draw[->, dashed, gray] (attn_b.east) -- (zoom_b.west);

\draw[->, ultra thick, red!60!black, dashed] (attn_s.east) -- (attn_b.west)
    node[midway, above, font=\footnotesize\bfseries, red!60!black] {Replace};

\end{tikzpicture}
\caption{Architecture overview. Left: standard Transformer with softmax attention. Right: Boltzmann Transformer, where only the attention mechanism is replaced by an Ising model with learnable couplings $J$. All other components remain identical.}
\label{fig:architecture}
\end{figure*}
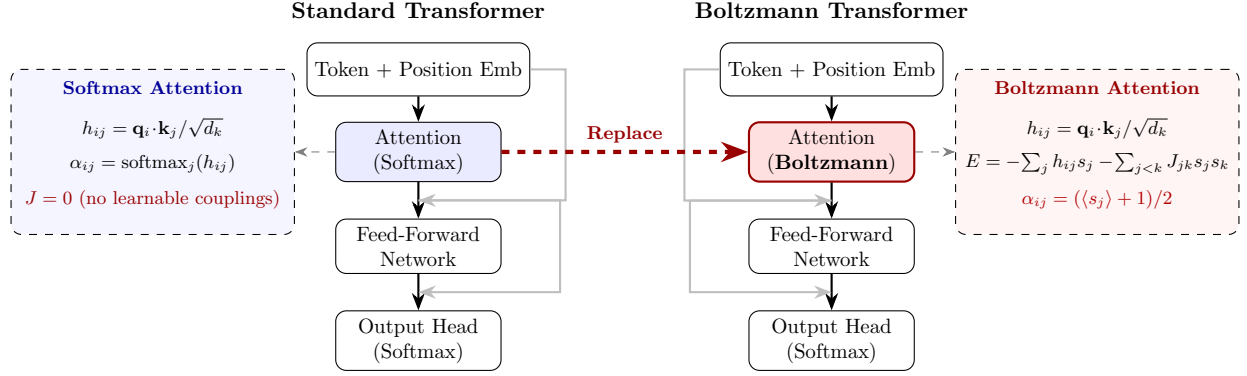

\subsection{Standard Attention as Non-Interacting Spins}

In standard multi-head attention with $H$ heads, the input sequence $\mathbf{x} \in \mathbb{R}^{T \times d_{\mathrm{model}}}$, where $d_{\mathrm{model}}$ is the embedding dimension, is projected into queries, keys, and values through learned weight matrices $W_Q, W_K, W_V \in \mathbb{R}^{d_{\mathrm{model}} \times d_k}$, with per-head dimension $d_k = d_{\mathrm{model}} / H$. Each head computes:
\begin{align}
    \mathbf{q} &= \mathbf{x} W_Q, \quad \mathbf{k} = \mathbf{x} W_K, \quad \mathbf{v} = \mathbf{x} W_V, \\
    \alpha_{ij} &= \mathrm{softmax}_j\!\left(\frac{\mathbf{q}_i \cdot \mathbf{k}_j}{\sqrt{d_k}}\right), \\
    \mathbf{o}_i &= \sum_j \alpha_{ij} \mathbf{v}_j,
\end{align}
where $\alpha_{ij} \in [0,1]$ are the normalized attention weights and $\mathbf{o}_i$ is the output at position $i$.

To show the absence of couplings explicitly, we assign a binary spin variable $s_j \in \{-1, +1\}$ to each key position $j$, where $s_j = +1$ means ``attend'' and $s_j = -1$ means ``ignore.'' In this language, the no-coupling ($J{=}0$) limit corresponds to the energy:
\begin{equation}
    E_i^{(J=0)}(\mathbf{s}) = -\sum_j h_{ij}\, s_j, \quad h_{ij} = \frac{\mathbf{q}_i \cdot \mathbf{k}_j}{\sqrt{d_k}},
    \label{eq:softmax_ising}
\end{equation}
The resulting Boltzmann distribution factorizes into independent spins:
\begin{equation}
    P(\mathbf{s}) = \prod_j P(s_j), \; P(s_j = +1) = \frac{\exp(h_{ij})}{\exp(h_{ij})+\exp(-h_{ij})} = \frac{1}{1+\exp(-2h_{ij})}=\sigma(2 h_{ij}).
\end{equation}
This is precisely sigmoid attention~\cite{ramapuram2024sigmoid}: each position's attend/ignore decision is statistically independent, and the connected correlation $\langle s_j s_k \rangle - \langle s_j \rangle \langle s_k \rangle$ vanishes for all $j \neq k$. Softmax attention shares the same local-field structure with $J=0$ but imposes a categorical normalization ($\sum_j \alpha_{ij} = 1$), which introduces competition between positions without learnable pairwise structure.

\subsection{Introducing Interactions}

We introduce interactions by adding learnable pairwise couplings between attention decisions. For each query position $i$, the attention pattern $\mathbf{s} \in \{-1,+1\}^T$ is governed by the Boltzmann distribution of the full Ising model:
\begin{equation}
    P_i(\mathbf{s}) = \frac{1}{Z_i} \exp\!\Bigl(\sum_j h_{ij} s_j + \sum_{j<k} J_{jk} s_j s_k \Bigr),
    \label{eq:boltzmann_attn}
\end{equation}
where $J_{jk} \in \mathbb{R}$ are learnable couplings for each head and shared across the batch, and $Z_i = \sum_{\mathbf{s}} \exp(-E_i(\mathbf{s}))$ is the partition function. Figure~\ref{fig:architecture} contrasts a standard Transformer block with our Boltzmann variant, where softmax attention is replaced by an energy-based interaction module.

The attention weight for position $j$ is obtained from marginal spin magnetization as
\begin{equation}
    \alpha_{ij} = \frac{\langle s_j \rangle_i + 1}{2} \in [0,1],
\end{equation}
where $\langle s_j \rangle_i = \sum_{\mathbf{s}} s_j P_i(\mathbf{s})$. Since these marginal activations do not sum to one in general, we use the normalized weights $\tilde{\alpha}_{ij} = \alpha_{ij} / \sum_k \alpha_{ik}$ for value aggregation.

The key difference from softmax is that $P_i(\mathbf{s})$ no longer factorizes. When $J_{jk} > 0$ (ferromagnetic), attending to position $j$ increases the probability of attending to $k$, so the model learns co-attention; when $J_{jk} < 0$ (antiferromagnetic), attending to $j$ decreases the probability of attending to $k$, producing competitive attention.

Thus, a position can receive attention not only through its own favorable local field, but also through couplings to other positions favored by the query. For example, when $J_{jk}>0$ and position $j$ is favored by its local field $h_{ij}>0$, the interaction term increases the tendency of $s_k$ to take the active state even when $h_{ik}\approx0$. In this case, position $k$ receives attention through its structural relationship to position $j$, rather than solely through its own query--key relevance. In softmax attention, by contrast, positions interact only through the normalizing denominator. For autoregressive models with causal masking, query position $i$ only attends to positions $j\leq i$, and the couplings $J_{jk}$ are applied only among causally visible positions.

\subsection{Cooperative Attention beyond Softmax}
\label{sec:expressiveness}

The Ising framework provides a precise characterization of how learnable couplings extend the representational capacity of attention beyond the $J{=}0$ regime.

\begin{proposition}[Fluctuation--dissipation identity {\cite{baxter1982exactly}}]
\label{prop:susceptibility}
For a query position $i$, the cross-susceptibility of the unnormalized Boltzmann marginal activation $\alpha_{ij} = (\langle s_j \rangle_i + 1)/2$ satisfies:
\begin{equation}
    \frac{\partial \alpha_{ij}}{\partial h_{ik}} = \frac{1}{2}\bigl(\langle s_j s_k \rangle_i - \langle s_j \rangle_i \langle s_k \rangle_i\bigr) \equiv \frac{1}{2}\langle s_j s_k \rangle_c.
    \label{eq:susceptibility}
\end{equation}
At $J{=}0$, spins are independent and $\langle s_j s_k \rangle_c = 0$ for $j \neq k$; at $J \neq 0$, the connected correlation is generally nonzero.
\end{proposition}
\begin{proof}
Since $\alpha_{ij} = (\langle s_j \rangle_i + 1)/2$, differentiating gives $\partial\alpha_{ij}/\partial h_{ik}=(1/2)\partial \langle s_j\rangle_i/\partial h_{ik}$. Using the standard Boltzmann identity $\partial \langle s_j\rangle_i/\partial h_{ik}=\langle s_j s_k\rangle_i-\langle s_j\rangle_i\langle s_k\rangle_i$ yields the result.
\end{proof}

We state this well-known identity to make the connection to attention explicit. Each coupling $J_{jk}$ controls the connected correlation between positions $j$ and $k$---a degree of freedom absent when $J{=}0$. As established in Section~\ref{sec:model}, both softmax and sigmoid attention operate in the $J{=}0$ regime where $\partial \alpha_{ik} / \partial h_{ij} = 0$ for $k \neq j$. When $J \neq 0$, Proposition~\ref{prop:susceptibility} gives $\partial \alpha_{ik} / \partial h_{ij} = \langle s_j s_k \rangle_c/2\neq 0$: the model acquires a \emph{cooperative} channel through which attending to one position can reinforce attention to another. Whether the correlation is positive (cooperative) or negative (competitive) depends on the sign and magnitude of the learned couplings; in our experiments, the learned couplings exhibit the distance-dependent structure shown in Section~\ref{sec:shakespeare}.
For a fixed query position, a non-interacting attention energy contains only $T$ local-field terms, one for each visible position. Boltzmann attention augments these with $\binom{T}{2}$ pairwise interaction terms, corresponding to the unordered pairs of visible positions. Thus, the number of interaction degrees of freedom absent from the $J=0$ model grows quadratically with the attention window size. This provides a simple structural explanation for why the benefit of learnable couplings can become more pronounced at longer sequence lengths, as examined empirically in the following section.


\section{Experiments}
\label{sec:experiments}

\subsection{Setup}
\label{sec:exact}

We evaluate the proposed method on two tasks. The first is \textbf{Tiny Shakespeare} character-level language modeling, using 100K characters with a 90\%/10\% train/validation split and 61 unique characters. Performance is measured by perplexity, $\mathrm{PPL}=\exp(\mathcal{L})$, where $\mathcal{L}$ is the cross-entropy loss; lower values indicate better performance. The second is a synthetic \textbf{bracket matching} task that requires pairwise attention coordination (Section~\ref{sec:bracket_task}) and is evaluated by classification accuracy. In both tasks, we compute the Boltzmann marginals (Eq.~\ref{eq:boltzmann_attn}) by exact enumeration over all $2^T$ spin configurations, introducing no approximation or sampling noise. The exponential cost limits this approach to moderate sequence lengths ($T \leq 14$ for Shakespeare, $T \leq 16$ for bracket matching), but exact computation provides a controlled proof of concept that isolates the effect of learnable $J$ without confounding artifacts. Scalable sampling-based alternatives to exact enumeration are explored in Section~\ref{sec:dqa}.

For language modeling, the model is a single-layer, causal (decoder-only) Transformer with embedding dimension $d_{\mathrm{model}} = 64$, a single attention head ($H = 1$, so $d_k = d_{\mathrm{model}} = 64$), a two-layer feed-forward network with feed-forward hidden dimension $d_{\mathrm{ff}} = 128$ and GELU activation, and dropout $0.1$. We use a single head ($H=1$) to isolate the effect of pairwise couplings: with only one head, improvements cannot be attributed to inter-head diversity, making the role of the intra-head couplings $J$ unambiguous. This architecture belongs to the same family underlying modern large language models~\cite{radford2019language,touvron2023llama}. We compare four attention modes: \textbf{softmax} (standard attention, Eq.~\ref{eq:softmax_attn}), \textbf{$h{+}J$} (full Ising model with learnable couplings, our proposed method), \textbf{$h$-only} ($J{=}0$, equivalent to sigmoid attention~\cite{ramapuram2024sigmoid}), and \textbf{$J$-only} ($h{=}0$, structural prior without data-dependent local fields).

All experiments use 10 random seeds. We optimize with AdamW~\cite{loshchilov2019decoupled} (a variant of Adam with decoupled weight decay regularization), using gradient clipping at 1.0 and early stopping with patience 20 (training halts if validation loss does not improve for 20 consecutive evaluation epochs). For Shakespeare, the learning rate is $10^{-3}$ for all parameters except the couplings $J$, which use a separate learning rate $\mathrm{LR}_J = 3 \times 10^{-5}$ with weight decay $0.01$. For bracket matching, the learning rate is $3 \times 10^{-4}$ and $\mathrm{LR}_J = 10^{-4}$ (weight decay $0.01$).

\subsection{Character-Level Language Modeling: Tiny Shakespeare}
\label{sec:shakespeare}

\begin{table}[t]
    \centering
    \caption{Tiny Shakespeare perplexity over 10 random seeds, reported as mean $\pm 1\sigma$. Boltzmann attention outperforms softmax for $T \geq 6$, with relative gains that generally increase with sequence length. Parenthetical value denotes relative improvement over softmax in percent; positive values indicate better performance.}
    \label{tab:shakespeare}
    \begin{tabular*}{\linewidth}{@{\extracolsep{\fill}}rcccc}
        \toprule
        $T$ & Softmax & $h{+}J$ & $h$-only ($J{=}0$) & $J$-only ($h{=}0$) \\
        \midrule
        4  & $6.775 \pm 0.031$ & $6.776 \pm 0.034$ \gap{$-$0.02} & $6.776 \pm 0.028$ \gap{$-$0.02} & $6.894 \pm 0.032$ \gap{$-$1.76} \\
        6  & $6.182 \pm 0.021$ & $\mathbf{6.156 \pm 0.024}$ \gap{+0.43} & $6.179 \pm 0.024$ \gap{+0.06} & $6.460 \pm 0.030$ \gap{$-$4.48} \\
        8  & $5.894 \pm 0.027$ & $\mathbf{5.846 \pm 0.036}$ \gap{+0.81} & $6.075 \pm 0.163$ \gap{$-$3.08} & $6.318 \pm 0.015$ \gap{$-$7.20} \\
        10  & $5.750 \pm 0.030$ & $\mathbf{5.697 \pm 0.031}$ \gap{+0.92} & $5.976 \pm 0.204$ \gap{$-$3.93} & $6.296 \pm 0.030$ \gap{$-$9.49} \\
        12  & $5.628 \pm 0.019$ & $\mathbf{5.567 \pm 0.026}$ \gap{+1.08} & $6.030 \pm 0.124$ \gap{$-$7.14} & $6.318 \pm 0.038$ \gap{$-$12.25} \\
        14  & $5.551 \pm 0.027$ & $\mathbf{5.493 \pm 0.025}$ \gap{+1.04} & $5.882 \pm 0.208$ \gap{$-$5.97} & $6.368 \pm 0.029$ \gap{$-$14.73} \\
        \bottomrule
    \end{tabular*}
\end{table}

Table~\ref{tab:shakespeare} shows the full four-way comparison across sequence lengths $T = 4$ to $14$. Two key findings emerge. First, the improvement grows with $T$: Boltzmann attention ($h{+}J$) matches softmax at $T = 4$ and increasingly outperforms it as $T$ grows---$+0.43\%$ at $T = 6$, $+0.81\%$ at $T = 8$, $+0.92\%$ at $T = 10$, $+1.08\%$ at $T = 12$, and $+1.04\%$ at $T = 14$. The rate plateaus between $T{=}12$ and $T{=}14$, possibly reflecting the limited capacity of the small model. This trend is expected: as the attention window expands, the number of position pairs whose co-attention structure the $J{=}0$ model cannot capture grows as $O(T^2)$, so the gap between interacting and non-interacting models widens with $T$.

Second, the results support that both $h$ and $J$ are necessary. The $h$-only ablation ($J = 0$, equivalent to sigmoid attention) performs \emph{worse} than softmax at $T \geq 8$, confirming that merely changing the functional form from categorical to independent binary spins does not help---the $J{=}0$ bottleneck remains. The $J$-only ablation ($h = 0$) performs substantially worse than all other modes, showing that data-dependent local fields are essential. Only the combination of $h$ and $J$ outperforms softmax, directly demonstrating that pairwise couplings---not the sigmoid parameterization---are the source of the improvement.

\begin{figure}[t]
    \centering
    \includegraphics[width=\linewidth]{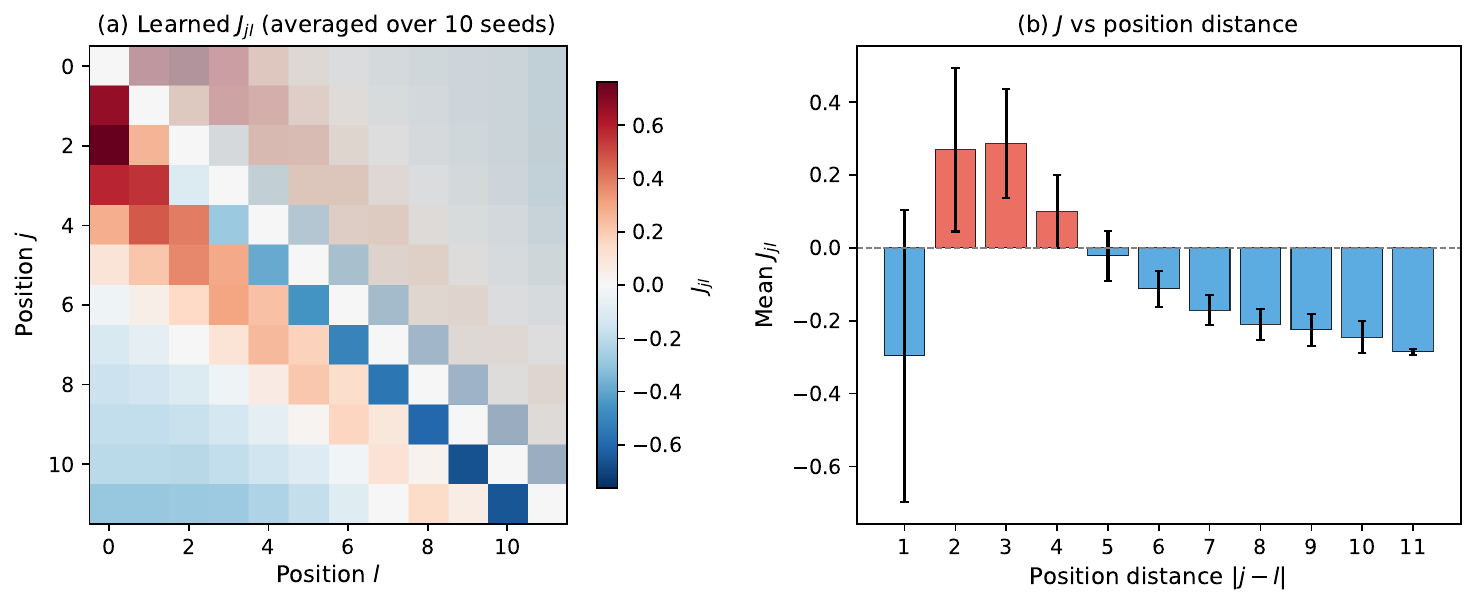}
    \caption{Learned pairwise couplings $J_{jl}$ on Tiny Shakespeare ($T = 12$, averaged over 10 seeds). (a)~Seed-averaged coupling matrix (lower triangle; upper triangle is masked by causal attention). (b)~Mean $J_{jl}$ as a function of position distance $|j - l|$. Nearby positions ($|j{-}l| = 2$--$4$) develop ferromagnetic (positive) couplings, promoting co-attention; distant positions ($|j{-}l| \geq 6$) develop antiferromagnetic (negative) couplings, suppressing co-attention. Error bars show $\pm 1\sigma$ across seeds and position pairs.}
    \label{fig:j_structure}
\end{figure}

Figure~\ref{fig:j_structure} visualizes the learned coupling matrix $J$ at $T = 12$. The model learns a distance-dependent structure: nearby positions ($|j{-}l| = 2$--$4$) develop positive (ferromagnetic) couplings that promote co-attention, while distant positions ($|j{-}l| \geq 6$) develop negative (antiferromagnetic) couplings that suppress co-attention. This pattern is consistent across seeds and resembles short-range ferromagnetic order with longer-range antiferromagnetic competition---a structure that $J{=}0$ attention cannot represent.

\subsection{Bracket Matching}
\label{sec:bracket_task}

The language modeling experiments demonstrate consistent improvements in a natural sequence-modeling setting. To examine the role of pairwise couplings more directly, we introduce a controlled synthetic task whose structure explicitly requires coordinated attention between positions. Specifically, we compare $J{=}0$ and $J{\neq}0$ attention on bracket matching, where correct prediction depends on identifying paired opening and closing brackets.

The input is a sequence of length $T$ containing Dyck-1 words~\cite{yao2021selfattention} (properly nested parentheses) interleaved with random filler tokens from a vocabulary of size 12. For each closing bracket at position $t$, the model must predict the position index of its matching opening bracket. Only closing brackets contribute to the loss.

The architecture is a single-layer causal Transformer with $d_{\mathrm{model}} = 32$, $H = 1$ head, and a two-layer feed-forward network with hidden dimension $d_{\mathrm{ff}} = 64$ and GELU activation---the same standard Transformer design used throughout this work. Resolving nested bracket matches requires the attention mechanism to track which opening brackets have already been consumed by inner closing brackets. This is inherently a \emph{pairwise} attention dependency: the correct attention target for one closing bracket depends on which targets other closing brackets have claimed. Without learnable pairwise couplings ($J{=}0$), the model has no mechanism to represent this constraint within the attention layer.

\begin{table}[t]
    \centering
    \caption{Bracket matching accuracy with the standard Transformer architecture ($d_{\mathrm{ff}}=64$), reported over 10 random seeds as mean $\pm 1\sigma$. Parenthetical values denote the gap relative to softmax in percentage points; positive values indicate better performance.}
    \label{tab:bracket}
    \begin{tabular*}{\linewidth}{@{\extracolsep{\fill}}rcccc}
        \toprule
        $T$ & Softmax & $h{+}J$ & $h$-only & $J$-only \\
        \midrule
        4  & $100.0 \pm 0.0$ & $100.0 \pm 0.0$ \gappp{+0.00} & $100.0 \pm 0.0$ \gappp{+0.00} & $100.0 \pm 0.0$ \gappp{+0.00} \\
        8  & $99.21 \pm 0.45$ & $99.00 \pm 0.53$ \gappp{$-$0.21} & $99.11 \pm 0.42$ \gappp{$-$0.10} & $100.00 \pm 0.01$ \gappp{+0.79} \\
        12 & $93.36 \pm 0.72$ & $\mathbf{94.91 \pm 1.71}$ \gappp{+1.55} & $92.65 \pm 0.47$ \gappp{$-$0.71} & $94.61 \pm 4.39$ \gappp{+1.25} \\
        16 & $93.08 \pm 1.17$ & $\mathbf{95.97 \pm 0.96}$ \gappp{+2.89} & $90.06 \pm 0.88$ \gappp{$-$3.02} & $87.42 \pm 4.47$ \gappp{$-$5.66} \\
        \bottomrule
    \end{tabular*}
\end{table}

Table~\ref{tab:bracket} shows the results. Boltzmann attention ($h{+}J$) outperforms softmax at longer sequence lengths, with the gap growing from $+1.55$\,pp at $T = 12$ to $+2.89$\,pp at $T = 16$. The gap persists even with the feed-forward network (FFN) present, confirming that learnable pairwise couplings provide a representational advantage that pointwise layers cannot replicate. This increasing gap with $T$ reflects the growing combinatorial complexity of nested bracket matching. Figure~\ref{fig:learning_curve} shows the learning curves at $T{=}12$: Boltzmann attention separates from softmax during training and maintains a persistent gap. Figure~\ref{fig:output_heatmap} (Appendix) visualizes this difference on a concrete example. The four-way ablation confirms the same pattern as in language modeling: $h$-only performance degrades at longer sequences ($-0.71$\,pp at $T{=}12$, $-3.02$\,pp at $T{=}16$), while $J$-only without data-dependent fields degrades substantially ($-5.66$\,pp at $T{=}16$).

\begin{figure}[t]
    \centering
    \includegraphics[width=0.65\linewidth]{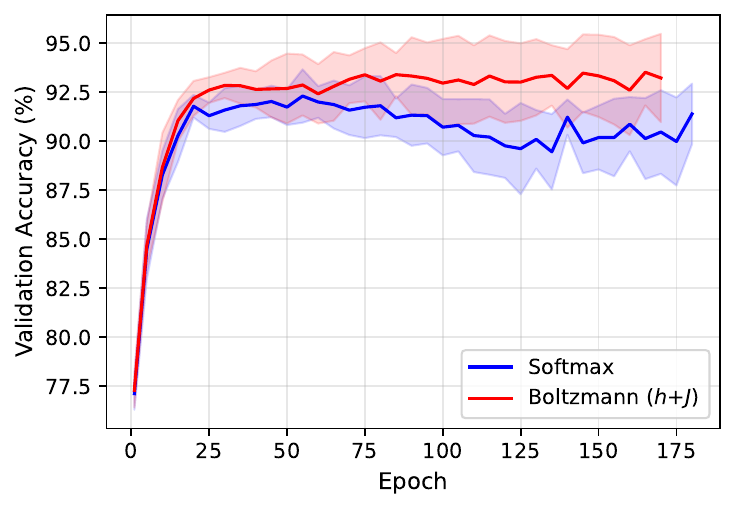}
    \caption{Learning curves for bracket matching at $T=12$, averaged over 10 random seeds with shaded regions indicating $\pm 1\sigma$. Boltzmann attention (red) separates from softmax attention (blue) during training and maintains a persistent accuracy gap.}
    \label{fig:learning_curve}
\end{figure}

A controlled ablation removing the FFN entirely (Appendix~\ref{app:bracket_noffn}) yields even larger gaps ($+4.53$\,pp at $T{=}16$), confirming that while the FFN can partially compensate for the absence of pairwise couplings, it cannot eliminate the $J{=}0$ bottleneck.

\subsection{Ablation: What Does $J$ Learn?}
\label{sec:ablation}

The four-way ablation results in Tables~\ref{tab:shakespeare}--\ref{tab:bracket} support three conclusions. First, $h$-only attention (sigmoid) performs comparably to or worse than softmax: removing $J$ limits performance regardless of whether the functional form is categorical (softmax) or independent binary (sigmoid). Second, $J$-only attention is limited: without data-dependent local fields, the model cannot leverage query--key similarity, though the computational implications of a pure structural prior merit further investigation (Section~\ref{sec:discussion}). Third, only the \emph{combination} of $h$ and $J$ outperforms softmax, as $J$ complements $h$ by introducing pairwise correlations that $h$ alone cannot produce.

\subsection{DQA for Training Boltzmann Attention}
\label{sec:dqa}

The preceding experiments establish that Boltzmann attention with learnable couplings consistently improves over softmax within a standard Transformer architecture. In those experiments, we used exact Boltzmann computation as a controlled implementation, allowing us to isolate the representational effect of the couplings without sampling artifacts. This choice, however, scales exponentially with the attention window size and is therefore suitable only for short sequences. To apply Boltzmann attention at practical sequence lengths, one needs an efficient sampling method that preserves the inter-position correlations induced by the Ising couplings. DQA provides a natural route to such sampling: prior studies have shown that DQA can generate approximate Boltzmann samples for Ising models at an effective inverse temperature controlled by the annealing rate~\cite{gyhm2024boltzmann}, and that such samples can support efficient end-to-end training of energy-based generative models~\cite{kim2026diabatic,kim2026BMVAE}.

As a proof of principle, we numerically simulate the DQA process by Trotterizing the time evolution into a quantum circuit. The simulation uses $n_{\mathrm{trotter}} = 200$ Trotter steps with an annealing time $\tau = 5.0$ ns, following the fast-annealing schedule of a current quantum annealing processor. During training, the Trotterized circuit produces attention-weight samples for each query position, and gradients are propagated through the sampling step via the straight-through estimator (STE). All other model components, including the embedding layer, FFN, output head, and hyperparameters, are identical to those used in the exact Boltzmann experiments in Sections~\ref{sec:shakespeare}--\ref{sec:bracket_task}.

\begin{table}[t]
    \centering
    \caption{Comparison of exact Boltzmann and DQA-trained Boltzmann attention ($h{+}J$). DQA training replaces exact enumeration with Trotterized quantum annealing simulation. Despite the approximation, DQA-trained models achieve competitive performance, validating DQA as a practical sampling method for Boltzmann attention. DQA results are single-seed.}
    \label{tab:dqa}
    \begin{tabular*}{\linewidth}{@{\extracolsep{\fill}}rcccc}
        \toprule
        & \multicolumn{2}{c}{Shakespeare (PPL $\downarrow$)} & \multicolumn{2}{c}{Bracket (Acc \% $\uparrow$)} \\
        \cmidrule(lr){2-3} \cmidrule(lr){4-5}
        $T$ & Exact & DQA & Exact & DQA \\
        \midrule
        4  & $6.776 \pm 0.034$ & $6.671$ & $100.0 \pm 0.0$ & $100.0$ \\
        6  & $6.156 \pm 0.024$ & $6.182$ & $100.0 \pm 0.0$ & $100.0$ \\
        8  & $5.846 \pm 0.036$ & $6.000$ & $99.00 \pm 0.53$ & $100.0$ \\
        \bottomrule
    \end{tabular*}
\end{table}

Table~\ref{tab:dqa} compares exact Boltzmann and DQA-trained models. On Shakespeare, DQA-trained models achieve perplexities comparable to exact computation across all sequence lengths. On bracket matching, DQA training achieves $100\%$ accuracy at both $T{=}4$ and $T{=}8$. These results suggest that the Trotterized DQA process produces samples of sufficient quality for end-to-end training of Boltzmann attention.

The significance of DQA-based training lies not in matching exact computation at small $T$, where exact enumeration is already tractable, but in connecting Boltzmann attention to a quantum-hardware sampling mechanism that scales linearly with $T$---each query position requires one anneal of $O(1)$ duration, yielding $O(T)$ total---in contrast to the $O(2^T)$ cost of exact classical enumeration. Together with prior demonstrations of DQA-trained Boltzmann machines~\cite{kim2026diabatic,kim2026BMVAE}, these results support DQA as a practical realization route for Boltzmann attention beyond exact enumeration, while retaining the representational advantage of learnable inter-position couplings.

\section{Discussion}
\label{sec:discussion}

Our results point to a general principle for when energy-based models improve neural network components. Two lines of evidence confirm that the $J{=}0$ bottleneck is structural rather than incidental: (i)~the advantage of learnable couplings persists within a standard Transformer architecture with an FFN (Table~\ref{tab:bracket}), and (ii)~the performance gap grows roughly monotonically with the sequence length $T$. The FFN applies the same pointwise transformation to each position independently, so it can enrich per-position representations but does not directly introduce the inter-position correlations provided by $J$. Together, these observations indicate that the bottleneck arises from the functional form of the attention operation itself and is not fully compensated by surrounding pointwise layers. This perspective suggests that other coupling-free operations in deep networks, such as independent gating in LSTMs or per-channel normalization, may also be candidates for Ising-type generalization.

Two recent works introduce pairwise couplings into attention but fail to improve over softmax (Table~\ref{tab:comparison}), each missing at least one of three necessary ingredients: freely learnable $J$, data-dependent $h$, and correlation-preserving inference (see Section~\ref{sec:related} for details).

We also investigate higher-order interactions beyond pairwise couplings, focusing on three-body terms; the full results are reported in Appendix~\ref{app:3body}, and we summarize the main takeaways here. On Shakespeare, adding three-body couplings $K$ does not improve over pairwise $h{+}J$, though $h{+}K$ nearly matches the full $h{+}J{+}K$, suggesting that $K$ can subsume the role of $J$ when combined with data-dependent fields. On bracket matching, however, $h{+}J{+}K$ and $h{+}K$ substantially outperform $h{+}J$ at $T \geq 12$, indicating that the combinatorial structure of nested brackets benefits from higher-order interactions in a way that natural language does not.

To quantify how the advantage evolves with sequence length, we fit $\mathrm{PPL}(T) = a\,e^{-bT} + c$ to the Shakespeare perplexity curves (Figure~\ref{fig:scaling_fit}). While the decay rates $b$ are similar across models (${\approx}0.28$--$0.30$), the asymptotic floor $c$ is consistently lower for Boltzmann variants ($h{+}J$: $5.44$, $h{+}J{+}K$: $5.45$) than for softmax ($5.51$) or $h$-only ($5.99$). This indicates that \emph{$J$ lowers the achievable performance floor} rather than accelerating convergence with $T$.

\begin{figure}[t]
    \centering
    \includegraphics[width=0.55\linewidth]{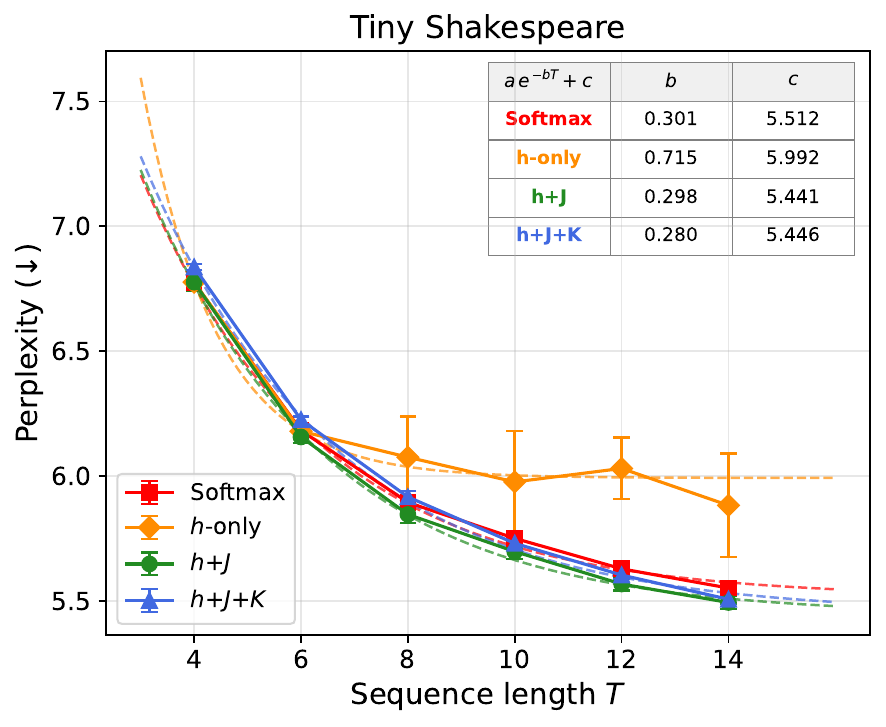}
    \caption{Sequence-length scaling of Tiny Shakespeare perplexity. Dashed curves show fits to $a e^{-bT}+c$, with the fitted parameters reported in the inset. The fitted asymptotic floor $c$ is lower for the Boltzmann variants ($h{+}J$: $5.44$, $h{+}J{+}K$: $5.45$) than for softmax ($5.51$) and $h$-only ($5.99$), suggesting that learnable interaction terms improve the attainable performance level. Error bars show $\pm 1\sigma$ over 10 random seeds.}
    \label{fig:scaling_fit}
\end{figure}

Our exact enumeration experiments use small models (1 layer, 1 head, $T \leq 16$) on character-level tasks, a deliberate choice to isolate the effect of $J$ without approximation artifacts. Scaling Boltzmann attention to practical sequence lengths requires inference methods that preserve inter-position correlations. The DQA results in Section~\ref{sec:dqa} validate one such path: Trotterized quantum annealing produces samples of sufficient quality for end-to-end training of Boltzmann attention, achieving competitive performance with exact computation (Table~\ref{tab:dqa}). On quantum hardware, DQA-based sampling~\cite{gyhm2024boltzmann,kim2026diabatic} scales linearly with $T$, removing the $O(2^T)$ bottleneck of exact enumeration while preserving the representational advantage of learnable couplings. This positions DQA as the practical realization method for Boltzmann attention at scale. Classical methods such as parallel tempering~\cite{hukushima1996exchange} can also exploit GPU parallelism for correlated sampling.

An intriguing direction arises from the $J$-only ablation. Although this variant underperforms on general tasks (Section~\ref{sec:ablation}), it has a distinct computational appeal: it eliminates the query--key computation entirely, since the attention pattern depends only on the learnable couplings $J$, not on input-dependent local fields. Crucially, the model retains input dependence through the value projection $\mathbf{v} = W_V \mathbf{x}$, so the attention output still adapts to the input even though the \emph{where-to-attend} decision is structural. In settings where positional structure dominates, this trade-off may be favorable, positioning $J$-only as a viable model rather than merely an ablation.

The $J$-only Ising model is particularly well-suited for DQA: because $h = 0$, the gradients with respect to $J$ can be estimated via correlation matching ($\partial \mathcal{L} / \partial J_{jl} \propto \langle s_j s_l \rangle_{\text{data}} - \langle s_j s_l \rangle_{\text{model}}$) using hardware samples alone, without backpropagation through the attention mechanism. For the full model with $h \neq 0$, DQA-based training is also possible but requires differentiating through the local fields, which introduces additional complexity. This suggests a QK-free, hardware-native attention paradigm in which a quantum annealer both performs inference and provides training gradients for the structural attention prior.

\section{Conclusion}
\label{sec:conclusion}

We showed that the absence of learnable pairwise couplings in standard attention---shared by both softmax and sigmoid despite their structural differences---is a representational bottleneck. Boltzmann attention addresses this by introducing learnable Ising couplings $J_{jk}$ that create inter-position correlations absent in the $J{=}0$ regime.

Exact Boltzmann computation confirms that introducing pairwise couplings consistently improves performance within a standard Transformer architecture, with the improvement growing roughly with sequence length: up to $+1.08\%$ in perplexity on Shakespeare at $T = 12$, and $+2.89$ percentage points in accuracy on bracket matching at $T = 16$. The feed-forward network cannot compensate for the $J{=}0$ bottleneck, and a controlled ablation without FFN (Appendix~\ref{app:bracket_noffn}) yields even larger gaps ($+4.53$\,pp).

These results establish that introducing learnable pairwise couplings into attention is a principled and empirically validated generalization. The absence of such couplings is a structural property of the attention operation itself, and learnable $J$ provides a direct, interpretable remedy.

Scaling beyond exact enumeration requires correlation-preserving sampling. We demonstrated that DQA serves as a practical realization method: Trotterized simulation achieves competitive performance with exact computation on both tasks (Table~\ref{tab:dqa}), confirming that approximate Boltzmann samples from a diabatic quantum process~\cite{gyhm2024boltzmann,kim2026diabatic} suffice for learning meaningful couplings. Because DQA cost grows only linearly with $T$ on quantum hardware, Boltzmann attention is not limited to the small-$T$ regime explored here. Classical methods such as parallel tempering~\cite{hukushima1996exchange} provide a complementary scaling path.

\section*{Acknowledgements}

This work was partly supported by the Institute of Information \& Communications Technology Planning \& Evaluation (IITP)-ITRC (Information Technology Research Center) grant funded by the Korea government (MSIT)(IITP-2026-RS-2026-25519864, $20\%$). This work was also supported by the IITP grant (No. 2019-0-00003, Research and Development of Core Technologies for Programming, Running, Implementing and Validating of Fault-Tolerant Quantum Computing System, $20\%$), the National Research Foundation of Korea (RS-2025-02309510, $20\%$), the Ministry of Trade, Industry, and Energy (MOTIE), Korea, under the Industrial Innovation Infrastructure Development Project (RS-2024-00466693, $20\%$), and by Korean ARPA-H Project through the Korea Health Industry Development Institute (KHIDI), funded by the Ministry of Health \& Welfare, Korea (RS-2025-25456722, $20\%$).

\appendix
\section*{Appendix}
\section{Output Probability Visualization}
\label{app:output_prob}
We illustrate the qualitative difference between softmax and Boltzmann attention on a concrete bracket matching example without an FFN (Figure~\ref{fig:output_heatmap}). In this setting, the model must identify the matching opening bracket for each closing bracket using the attention mechanism alone, without downstream pointwise compensation. Softmax often assigns probability mass to nearby but incorrect positions, leading to several wrong argmax predictions. In contrast, Boltzmann attention produces sharper probability mass around the correct matching positions, correctly identifying all pairs in this example. This visualization supports the quantitative results in Appendix~\ref{app:bracket_noffn}, where removing the FFN amplifies the performance gap between the two attention mechanisms.

\begin{figure*}[htbp]
    \centering
    \includegraphics[width=0.95\linewidth]{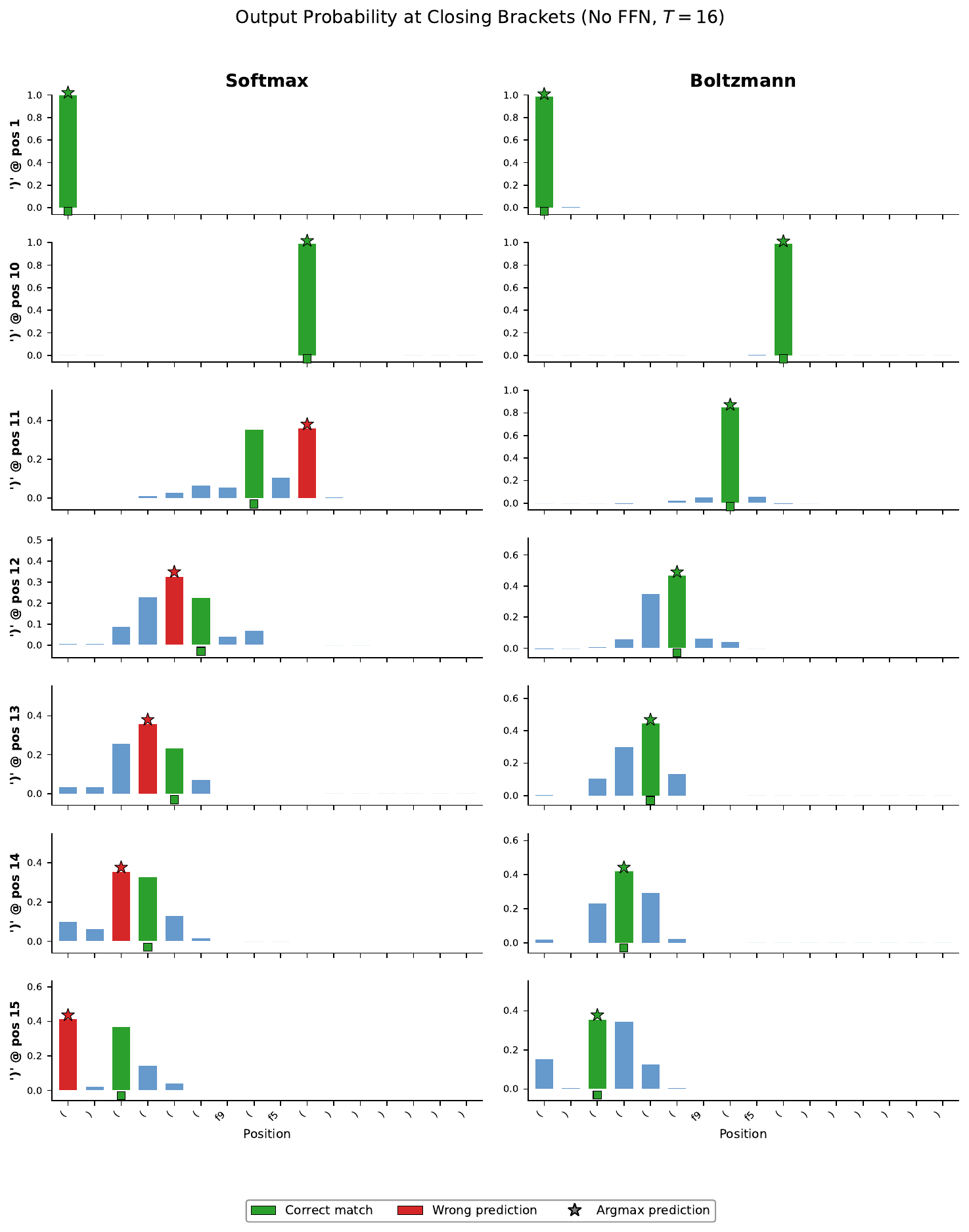}
    \caption{Output probability distribution at each closing bracket for the no-FFN bracket matching model ($T=16$, $H=1$). Green bars mark the correct match position; red bars mark the incorrect argmax prediction. Softmax (left) fails on 5 of 7 closing brackets, while Boltzmann (right) correctly identifies all matching positions.}
    \label{fig:output_heatmap}
\end{figure*}

\section{Bracket Matching Without Feed-Forward Network}
\label{app:bracket_noffn}

To isolate the attention mechanism from other model components, we also evaluate bracket matching with a minimal architecture that removes the feed-forward network entirely: a single-layer causal Transformer with $d_{\mathrm{model}} = 32$ and $H = 1$ head, so that attention is the sole computation pathway. Any performance gap must originate from the attention formulation itself, not from downstream compensation.

\begin{table}[h]
    \centering
    \caption{Bracket matching accuracy (\%, 10 seeds, $\pm 1\sigma$) \emph{without} FFN. $\Delta$: gap vs.\ softmax (pp). The larger gaps compared to the standard architecture (Table~\ref{tab:bracket}) confirm that the FFN partially compensates for the $J{=}0$ bottleneck but cannot eliminate it.}
    \label{tab:bracket_noffn}
    \begin{tabular*}{\linewidth}{@{\extracolsep{\fill}}rcccc}
        \toprule
        $T$ & Softmax & $h{+}J$ & $h$-only & $J$-only \\
        \midrule
        4  & $100.00 \pm 0.00$ & $100.00 \pm 0.00$ \gappp{+0.00} & $100.00 \pm 0.00$ \gappp{+0.00} & $100.00 \pm 0.00$ \gappp{+0.00} \\
        8  & $98.83 \pm 0.24$  & $\mathbf{99.26 \pm 0.41}$ \gappp{+0.43} & $98.61 \pm 0.24$ \gappp{$-$0.22} & $93.34 \pm 1.15$ \gappp{$-$5.49} \\
        12 & $91.19 \pm 0.59$  & $\mathbf{93.29 \pm 0.43}$ \gappp{+2.10} & $90.93 \pm 0.48$ \gappp{$-$0.26} & $84.65 \pm 2.38$ \gappp{$-$6.54} \\
        16 & $85.13 \pm 0.49$  & $\mathbf{89.65 \pm 0.61}$ \gappp{+4.53} & $85.78 \pm 0.47$ \gappp{+0.65} & $76.65 \pm 0.56$ \gappp{$-$8.48} \\
        \bottomrule
    \end{tabular*}
\end{table}

Without FFN, the performance gap between Boltzmann and softmax attention is substantially larger: $+0.43$\,pp at $T{=}8$, $+2.10$\,pp at $T{=}12$, and $+4.53$\,pp at $T{=}16$ (compared to $+1.55$\,pp and $+2.89$\,pp with FFN in Table~\ref{tab:bracket}). This confirms that the feed-forward network can partially compensate for the $J{=}0$ bottleneck through pointwise transformations, but cannot replicate the inter-position correlations provided by learnable couplings.

\section{Higher-Order Interactions}
\label{app:3body}

We investigate whether extending the Ising energy from pairwise to three-body interactions improves performance. The extended energy function is
\[
    E_i(\mathbf{s}) = -\sum_j h_{ij}\, s_j - \sum_{j<k} J_{jk}\, s_j s_k - \sum_{j<k<l} K_{jkl}\, s_j s_k s_l,
\]
where $K_{jkl}$ is a learnable symmetric three-body coupling tensor with $\binom{T}{3}$ free parameters. The distribution remains Boltzmann ($P_i(\mathbf{s}) \propto e^{-E_i(\mathbf{s})}$), and exact enumeration is unchanged ($2^T$ configurations). The additional cost is polynomial: each configuration requires $O(T^3)$ triplet energy terms instead of $O(T^2)$ pairwise terms. We test four different modes---full three-body ($h{+}J{+}K$), $h{+}K$ ($J{=}0$), $J{+}K$ ($h{=}0$), and $K$-only ($h{=}0$, $J{=}0$)---providing a systematic ablation of the role of each energy term.

\begin{table}[h]
    \centering
    \caption{Higher-order interactions on Tiny Shakespeare (PPL $\downarrow$, 10 seeds). Among $h$-free models, $J{+}K$ outperforms both $J$-only and $K$-only; $h{+}K$ nearly matches the full $h{+}J{+}K$.}
    \label{tab:3body_shakespeare}
    \begin{tabular*}{\linewidth}{@{\extracolsep{\fill}}rcccccccc}
        \toprule
        $T$ & Softmax & $h$-only & $J$-only & $K$-only & $h{+}J$ & $h{+}K$ & $J{+}K$ & $h{+}J{+}K$ \\
        \midrule
        4  & $6.775$ & $6.776$ & $6.894$ & $6.952$ & $6.776$ & $6.840$ & $6.940$ & $6.837$ \\
        6  & $6.182$ & $6.179$ & $6.460$ & $6.417$ & $\mathbf{6.156}$ & $6.228$ & $6.400$ & $6.225$ \\
        8  & $5.894$ & $6.075$ & $6.318$ & $6.168$ & $\mathbf{5.846}$ & $5.915$ & $6.173$ & $5.916$ \\
        10 & $5.750$ & $5.976$ & $6.296$ & $6.053$ & $\mathbf{5.697}$ & $5.733$ & $6.017$ & $5.730$ \\
        12 & $5.628$ & $6.030$ & $6.318$ & $5.933$ & $\mathbf{5.567}$ & $5.608$ & $5.928$ & $5.603$ \\
        14 & $5.551$ & $5.882$ & $6.368$ & $5.847$ & $\mathbf{5.493}$ & $5.512$ & $5.869$ & $5.506$ \\
        \bottomrule
    \end{tabular*}
\end{table}

Table~\ref{tab:3body_shakespeare} shows that three-body couplings consistently underperform the pairwise model ($h{+}J$) on Shakespeare, though the gap narrows with increasing $T$ (from $0.061$ PPL at $T{=}4$ to $0.013$ at $T{=}14$). Notably, $h{+}K$ ($J{=}0$) performs nearly identically to the full $h{+}J{+}K$ (gap $\leq 0.006$ PPL across all $T$), indicating that when three-body couplings and data-dependent fields are both present, pairwise couplings become largely redundant. Without data-dependent fields ($h{=}0$), all models substantially underperform, confirming that structural couplings alone cannot substitute for input-dependent attention.

\begin{table}[h]
    \centering
    \caption{Higher-order interactions on bracket matching with FFN (accuracy \%, 10 seeds). Two-body columns reproduce Table~\ref{tab:bracket} for comparison. $h{+}J{+}K$ and $h{+}K$ substantially outperform $h{+}J$ at $T \geq 12$, unlike on Shakespeare.}
    \label{tab:3body_bracket}
    \begin{tabular*}{\linewidth}{@{\extracolsep{\fill}}rcccccccc}
        \toprule
        $T$ & Softmax & $h$-only & $J$-only & $K$-only & $h{+}J$ & $h{+}K$ & $J{+}K$ & $h{+}J{+}K$ \\
        \midrule
        4  & $100.0$ & $100.0$ & $100.0$ & $100.0$ & $100.0$ & $100.0$ & $100.0$ & $100.0$ \\
        8  & $99.21$ & $99.11$ & $100.00$ & $100.00$ & $99.00$ & $99.37$ & $99.96$ & $99.46$ \\
        12 & $93.36$ & $92.65$ & $94.61$ & $95.97$ & $94.91$ & $99.04$ & $96.24$ & $\mathbf{99.84}$ \\
        16 & $93.08$ & $90.06$ & $87.42$ & $89.91$ & $95.97$ & $98.16$ & $91.31$ & $\mathbf{98.26}$ \\
        \bottomrule
    \end{tabular*}
\end{table}

Table~\ref{tab:3body_bracket} shows a strikingly different pattern from Shakespeare: on bracket matching, three-body models with data-dependent fields ($h{+}J{+}K$ and $h{+}K$) substantially outperform the pairwise $h{+}J$ at longer sequences ($+4.93$\,pp and $+4.13$\,pp at $T{=}12$; $+2.29$\,pp and $+2.19$\,pp at $T{=}16$). This suggests that the combinatorial structure of nested bracket matching benefits from higher-order interactions in a way that natural language does not. As on Shakespeare, $h{+}K$ performs comparably to $h{+}J{+}K$, confirming that pairwise couplings become redundant when three-body couplings and data-dependent fields are both present.
\bibliographystyle{unsrt}
\bibliography{references}

@inproceedings{devlin2019bert,
  author       = {Jacob Devlin and
                  Ming{-}Wei Chang and
                  Kenton Lee and
                  Kristina Toutanova},
  title        = {{BERT:} Pre-training of Deep Bidirectional Transformers for Language
                  Understanding},
  booktitle    = {Proceedings of the 2019 Conference of the North American Chapter of
                  the Association for Computational Linguistics: Human Language Technologies,
                  {NAACL-HLT} 2019, Minneapolis, MN, USA},
  volume = {1},
  pages        = {4171--4186},
  publisher    = {Association for Computational Linguistics},
  year         = {2019},
  url          = {https://doi.org/10.18653/v1/n19-1423},
  doi          = {10.18653/v1/N19-1423},
}

@inproceedings{dosovitskiy2021image,
  author       = {Alexey Dosovitskiy and
                  Lucas Beyer and
                  Alexander Kolesnikov and
                  Dirk Weissenborn and
                  Xiaohua Zhai and
                  Thomas Unterthiner and
                  Mostafa Dehghani and
                  Matthias Minderer and
                  Georg Heigold and
                  Sylvain Gelly and
                  Jakob Uszkoreit and
                  Neil Houlsby},
  title        = {An Image is Worth 16x16 Words: Transformers for Image Recognition
                  at Scale},
  booktitle    = {9th International Conference on Learning Representations, Austria, May 3-7},
  year         = {2021},
  url          = {https://openreview.net/forum?id=YicbFdNTTy},
  doi          = {10.48550/arXiv.2010.11929},
}

@InProceedings{radford2021learning,
  title = 	 {Learning Transferable Visual Models From Natural Language Supervision},
  author =       {Radford, Alec and Kim, Jong Wook and Hallacy, Chris and Ramesh, Aditya and Goh, Gabriel and Agarwal, Sandhini and Sastry, Girish and Askell, Amanda and Mishkin, Pamela and Clark, Jack and Krueger, Gretchen and Sutskever, Ilya},
  booktitle = 	 {Proceedings of the 38th International Conference on Machine Learning},
  pages = 	 {8748--8763},
  year = 	 {2021},
  editor = 	 {Meila, Marina and Zhang, Tong},
  volume = 	 {139},
  series = 	 {Proceedings of Machine Learning Research},
  month = 	 {18--24 Jul},
  publisher =    {PMLR},
  url = 	 {https://proceedings.mlr.press/v139/radford21a.html},
}

@inproceedings{loshchilov2019decoupled,
  title     = {Decoupled Weight Decay Regularization},
  author    = {Loshchilov, Ilya and Hutter, Frank},
  booktitle = {International Conference on Learning Representations},
  year      = {2019},
  doi       = {10.48550/arXiv.1711.05101},
}

@inproceedings{vaswani2017attention,
  title     = {Attention is All You Need},
  author    = {Vaswani, Ashish and Shazeer, Noam and Parmar, Niki and Uszkoreit, Jakob and Jones, Llion and Gomez, Aidan N and Kaiser, {\L}ukasz and Polosukhin, Illia},
  booktitle = {Advances in Neural Information Processing Systems},
  volume    = {30},
  year      = {2017},
  doi       = {10.5555/3295222.3295349},
}

@inproceedings{ramsauer2020hopfield,
  title     = {Hopfield Networks is All You Need},
  author    = {Ramsauer, Hubert and Sch{\"a}fl, Bernhard and Lehner, Johannes and Seidl, Philipp and Widrich, Michael and Adler, Thomas and Gruber, Lukas and Holzleitner, Markus and Kreil, David P and Kopp, Michael K and Klambauer, G{\"u}nter and Brandstetter, Johannes and Hochreiter, Sepp},
  booktitle = {International Conference on Learning Representations},
  year      = {2021},
  doi       = {10.48550/arXiv.2008.02217},
}

@inproceedings{krotov2020large,
  title     = {Large Associative Memory Problem in Neurobiology and Machine Learning},
  author    = {Krotov, Dmitry and Hopfield, John J},
  booktitle = {International Conference on Learning Representations},
  year      = {2021},
  doi       = {10.48550/arXiv.2008.06996},
}

@article{ota2023attnbm,
  title   = {Attention in a Family of {B}oltzmann Machines Emerging from Modern {H}opfield Networks},
  author  = {Ota, Toshihiro and Karakida, Ryo},
  journal = {Neural Computation},
  volume  = {35},
  number  = {8},
  pages   = {1463--1480},
  year    = {2023},
  doi     = {10.1162/neco_a_01597},
}

@inproceedings{hoover2023energy,
  title     = {Energy Transformer},
  author    = {Hoover, Benjamin and Liang, Yuchen and Pham, Bao and Panda, Rameswar and Strobelt, Hendrik and Chau, Duen Horng and Zaki, Mohammed J and Krotov, Dmitry},
  booktitle = {Advances in Neural Information Processing Systems},
  year      = {2023},
  doi       = {10.5555/3666122.3667319},
}

@article{gyhm2024boltzmann,
  title   = {Boltzmann Sampling by Diabatic Quantum Annealing},
  author  = {Gyhm, Ju-Yeon and Kim, Gilhan and Kwon, Hyukjoon and Baek, Yongjoo},
  journal = {Physical Review E},
  volume  = {113},
  pages   = {065302},
  year    = {2026},
  doi     = {10.1103/9bc4-ddkc},
}

@article{kim2026diabatic,
  title   = {Diabatic Quantum Annealing for Training Energy-Based Generative Models},
  author  = {Kim, Gilhan and Gyhm, Ju-Yeon and Park, Daniel K},
  journal = {Physical Review E},
  volume  = {113},
  pages   = {035302},
  year    = {2026},
  doi     = {10.1103/2g6m-whm2},
}

@article{kim2026BMVAE,
  title   = {Multi-Mode Quantum Annealing for Generative Representation Learning with Boltzmann Priors},
  author  = {Kim, Gilhan and Park, Daniel K},
  journal = {arXiv preprint arXiv:2604.00919},
  year    = {2026},
  doi     = {10.48550/arXiv.2604.00919},
}

@article{du2025qama,
  title   = {{QAMA}: Scalable Quantum Annealing Multi-Head Attention Operator for Deep Learning},
  author  = {Du, Peng and Shi, Jinjing and Wang, Wenxuan and Ma, Yin and Wen, Kai and Li, Xuelong},
  journal = {arXiv preprint arXiv:2504.11083},
  year    = {2025},
  doi     = {10.48550/arXiv.2504.11083},
}

@misc{mcbal2023spin,
  title        = {Spin-Model Transformers},
  author       = {Bal, Matthias},
  year         = {2023},
  howpublished = {\url{https://mcbal.github.io/post/spin-model-transformers/}}
}

@article{ackley1985learning,
  title   = {A Learning Algorithm for {B}oltzmann Machines},
  author  = {Ackley, David H and Hinton, Geoffrey E and Sejnowski, Terrence J},
  journal = {Cognitive Science},
  volume  = {9},
  number  = {1},
  pages   = {147--169},
  year    = {1985},
  doi     = {10.1207/s15516709cog0901_7},
}

@article{hinton2002training,
  title   = {Training Products of Experts by Minimizing Contrastive Divergence},
  author  = {Hinton, Geoffrey E},
  journal = {Neural Computation},
  volume  = {14},
  number  = {8},
  pages   = {1771--1800},
  year    = {2002},
  doi     = {10.1162/089976602760128018},
}

@article{hukushima1996exchange,
  title   = {Exchange {M}onte {C}arlo Method and Application to Spin Glass Simulations},
  author  = {Hukushima, Koji and Nemoto, Kazuyuki},
  journal = {Journal of the Physical Society of Japan},
  volume  = {65},
  number  = {6},
  pages   = {1604--1608},
  year    = {1996},
  doi     = {10.1143/JPSJ.65.1604},
}

@inproceedings{katharopoulos2020transformers,
  title     = {Transformers are {RNN}s: Fast Autoregressive Transformers with Linear Attention},
  author    = {Katharopoulos, Angelos and Vyas, Apoorv and Pappas, Nikolaos and Fleuret, Fran{\c{c}}ois},
  booktitle = {International Conference on Machine Learning},
  year      = {2020},
  doi       = {10.5555/3524938.3525416},
}

@inproceedings{yao2021selfattention,
  title     = {Self-Attention Networks Can Process Bounded Hierarchical Languages},
  author    = {Yao, Shunyu and Peng, Binghui and Papadimitriou, Christos and Narasimhan, Karthik},
  booktitle = {Annual Meeting of the Association for Computational Linguistics},
  year      = {2021},
  doi       = {10.18653/v1/2021.acl-long.292},
}

@inproceedings{ramapuram2024sigmoid,
  title     = {Theory, Analysis, and Best Practices for Sigmoid Self-Attention},
  author    = {Ramapuram, Jason and Danieli, Federico and Dhekane, Eeshan Gunesh and Weers, Floris and Busbridge, Dan and Ablin, Pierre and Likhomanenko, Tatiana and Digani, Jagrit and Gu, Zijin and Shidani, Amitis and Webb, Russ},
  booktitle = {International Conference on Learning Representations},
  year      = {2025},
  doi       = {10.48550/arXiv.2409.04431},
}

@article{yan2025sigmoid,
  title   = {Sigmoid Self-Attention has Lower Sample Complexity than Softmax Self-Attention: A Mixture-of-Experts Perspective},
  author  = {Yan, Fanqi and Nguyen, Huy and Akbarian, Pedram and Ho, Nhat and Rinaldo, Alessandro},
  journal = {arXiv preprint arXiv:2502.00281},
  year    = {2025},
  doi     = {10.48550/arXiv.2502.00281},
}

@inproceedings{bahdanau2015neural,
  title     = {Neural Machine Translation by Jointly Learning to Align and Translate},
  author    = {Bahdanau, Dzmitry and Cho, Kyunghyun and Bengio, Yoshua},
  booktitle = {International Conference on Learning Representations},
  year      = {2015},
  doi       = {10.48550/arXiv.1409.0473},
}

@article{hopfield1982neural,
  title   = {Neural Networks and Physical Systems with Emergent Collective Computational Abilities},
  author  = {Hopfield, John J},
  journal = {Proceedings of the National Academy of Sciences},
  volume  = {79},
  number  = {8},
  pages   = {2554--2558},
  year    = {1982},
  doi     = {10.1073/pnas.79.8.2554},
}

@article{demircigil2017model,
  title   = {On a Model of Associative Memory with Huge Storage Capacity},
  author  = {Demircigil, Mete and Heusel, Judith and L{\"o}we, Matthias and Upgang, Sven and Vermet, Franck},
  journal = {Journal of Statistical Physics},
  volume  = {168},
  pages   = {288--299},
  year    = {2017},
  doi     = {10.1007/s10955-017-1806-y},
}

@inproceedings{kitaev2020reformer,
  title     = {Reformer: The Efficient Transformer},
  author    = {Kitaev, Nikita and Kaiser, {\L}ukasz and Levskaya, Anselm},
  booktitle = {International Conference on Learning Representations},
  year      = {2020},
  doi       = {10.48550/arXiv.2001.04451},
}

@inproceedings{choromanski2021rethinking,
  title     = {Rethinking Attention with Performers},
  author    = {Choromanski, Krzysztof and Likhosherstov, Valerii and Dohan, David and Song, Xingyou and Gane, Andreea and Sarlos, Tamas and Hawkins, Peter and Davis, Jared and Mohiuddin, Afroz and Kaiser, Lukasz and Belanger, David and Colwell, Lucy and Weller, Adrian},
  booktitle = {International Conference on Learning Representations},
  year      = {2021},
  doi       = {10.48550/arXiv.2009.14794},
}

@article{kadowaki1998quantum,
  title   = {Quantum Annealing in the Transverse {I}sing Model},
  author  = {Kadowaki, Tadashi and Nishimori, Hidetoshi},
  journal = {Physical Review E},
  volume  = {58},
  number  = {5},
  pages   = {5355--5363},
  year    = {1998},
  doi     = {10.1103/PhysRevE.58.5355},
}

@article{johnson2011quantum,
  title   = {Quantum Annealing with Manufactured Spins},
  author  = {Johnson, Mark W and Amin, Mohammad H S and Gildert, Suzanne and Lanting, Trevor and Hamze, Firas and Dickson, Neil and Harris, Richard and Berkley, Andrew J and Johansson, Jan and Bunyk, Paul and others},
  journal = {Nature},
  volume  = {473},
  number  = {7346},
  pages   = {194--198},
  year    = {2011},
  doi     = {10.1038/nature10012},
}

@article{child2019generating,
  title   = {Generating Long Sequences with Sparse Transformers},
  author  = {Child, Rewon and Gray, Scott and Radford, Alec and Sutskever, Ilya},
  journal = {arXiv preprint arXiv:1904.10509},
  year    = {2019},
  doi     = {10.48550/arXiv.1904.10509},
}

@article{poclopez2024dynamical,
  title   = {Dynamical Mean-Field Theory of Self-Attention Neural Networks},
  author  = {Poc-L{\'o}pez, {\'A}ngel and Aguilera, Miguel},
  journal = {arXiv preprint arXiv:2406.07247},
  year    = {2024},
  doi     = {10.48550/arXiv.2406.07247},
}

@techreport{radford2019language,
  title       = {Language Models are Unsupervised Multitask Learners},
  author      = {Radford, Alec and Wu, Jeffrey and Child, Rewon and Luan, David and Amodei, Dario and Sutskever, Ilya},
  institution = {OpenAI},
  year        = {2019}
}

@article{touvron2023llama,
  title   = {{LLaMA}: Open and Efficient Foundation Language Models},
  author  = {Touvron, Hugo and Lavril, Thibaut and Izacard, Gautier and Martinet, Xavier and Lachaux, Marie-Anne and Lacroix, Timoth{\'e}e and Rozi{\`e}re, Baptiste and Goyal, Naman and Hambro, Eric and Azhar, Faisal and others},
  journal = {arXiv preprint arXiv:2302.13971},
  year    = {2023},
  doi     = {10.48550/arXiv.2302.13971},
}

@book{baxter1982exactly,
  title     = {Exactly Solved Models in Statistical Mechanics},
  author    = {Baxter, Rodney J},
  year      = {1982},
  publisher = {Academic Press}
}

\end{document}